\documentclass{article}

\usepackage[final, nonatbib]{neurips_2020}
\usepackage{verbatim}

\usepackage[utf8]{inputenc} 
\usepackage[T1]{fontenc}    
\usepackage{url}            
\usepackage{booktabs}       
\usepackage{amsfonts}       
\usepackage{nicefrac}       
\usepackage{microtype}      
\usepackage{amsmath}
\usepackage{amssymb}
\usepackage{gensymb}
\usepackage{amsthm}
\usepackage{adjustbox}
\usepackage{subfigure}
\usepackage{xcolor}
\newtheorem{theorem}{Theorem}
\newtheorem{assumption}{Assumption}
\title{Domain Generalization for Medical Imaging Classification with Linear-Dependency Regularization}

\author{Haoliang Li$^1$ \quad YuFei Wang$^1$ \quad Renjie Wan$^1$ \quad Shiqi Wang$^2$ \quad 
Tie-Qiang Li$^{3,4}$ \quad Alex C. Kot$^1$ \\
$^1$Rapid-Rich Object Search Lab, Nanyang Technological University, Singapore \\
$^2$Department of Computer Science, City University of Hong Kong, China \\
$^3$Department of Clinical Science, Intervention, and Technology, Karolinska Institute, Sweden\\
$^4$Department of Medical Radiation and Nuclear Medicine, Karolinska Uni­versity Hospital, Sweden\\
\{lihaoliang,yufei001,rjwan,eackot\}@ntu.edu.sg \quad shiqwang@cityu.edu.hk \quad tie-qiang.li@ki.se
}

\begin{document}
	
	\maketitle
	
	\begin{abstract}
		Recently, we have witnessed great progress in the field of medical imaging classification by adopting deep neural networks. However, the recent advanced models still require accessing sufficiently large and representative datasets for training, which is often unfeasible in clinically realistic environments. When trained on limited datasets, the deep neural network is lack of generalization capability, as the trained deep neural network on data within a certain distribution (e.g. the data captured by a certain device vendor or patient population) may not be able to generalize to the data with another distribution. 
		In this paper, we introduce a simple but effective approach to improve the generalization capability of deep neural networks in the field of medical imaging classification. Motivated by the observation that the domain variability of the medical images is to some extent compact, we propose to learn a representative feature space through variational encoding with a novel linear-dependency regularization term to capture the shareable information among medical data collected from different domains. As a result, the trained neural network is expected to equip with better generalization capability to the ``unseen" medical data.   Experimental results on two challenging medical imaging classification tasks indicate that our method can achieve better cross-domain generalization capability compared with state-of-the-art baselines.
	\end{abstract}
	

\section{Introduction}
Due to the breakthrough in machine learning and deep learning, recent years have witnessed numerous significant successes in various medical imaging tasks.  However, one of the limitations of deep learning is that it lacks generalization capability when the number of training data is not sufficient \cite{zhang2018translating}. In practice, it is often the case that the testing data (a.k.a. target domain) can be dissimilar to the training data (a.k.a. source domain) in terms of many factors, such as imaging protocol, device vendors and patient populations. Such domain shift problem can lead to a significantly negative impact on the performance of medical imaging classification. To tackle such domain shift problem, domain adaptation \cite{pan2010survey} aims to transfer the knowledge from a source domain to a different but relevant target domain.  Recently, many studies have been conducted to improve the transferable capability in the field of medical imaging classification with domain adaptation by assuming that target domain data are accessible \cite{zhang2018task,dou2018pnp}.

In many cases, requiring to access the target domain data in advance may not be feasible. For example, in the real-time clinical application scenario, it is difficult to collect sufficient target domain data to help with network training. For another example, it is also difficult to access the target domain data as many medical data are protected by privacy regulation.  Thus, it is natural to ask whether we can still learn a generalized deep neural network without any prior knowledge regarding the target domain. Domain generalization has been proposed to tackle this problem by assuming to have no access to the target information but utilizing multiple source domains' information to better generalize to the ``unseen" new domain for testing.

Generally speaking, current research regarding domain generalization in the field of medical imaging classification can be categorized into two streams. The first stream aims at conducting data augmentation based on medical imaging data in terms of image quality, image appearance and spatial shape \cite{zhang2019unseen}. Although the variation of medical images turns out to be more compact as the capturing environment can be fixed in advanced compared with the images captured in our daily life,  it may be difficult to choose suitable augmentation types and magnitudes for clinical deployment purposes in a certain environment.  The other stream leverages the advantage of domain alignment or meta-learning methods for feature representation learning \cite{yoon2019generalizable,dou2019domain}. However, the learned feature representation may still suffer from the overfitting problem, as the feature representations are only shareable among multiple source domains which may not be able to generalize to target. 
 
  In this work, we propose to marriage the advantage of data augmentation and domain alignment to tackle the domain generalization problem for medical imaging classification.  Instead of directly conducting augmentation in the image domain through some linear transformations with pre-defined parameters \cite{zhang2019unseen}, we assume that there exists linear dependency in a latent space among various domains. To model such linear dependency,  we propose to train a deep neural network with a novel rank regularization term on latent feature space by setting the rank of latent feature to be the number of categories. Meanwhile, we also propose to restrict the distribution of latent features to follow a pre-defined prior distribution through variational encoding. We theoretically prove that an upper bound on the empirical risk of any ``unseen" but related target domain can be achieved under our formulation, such that the overfitting problem can be alleviated.  Experimental results on two challenging medical imaging classification tasks, including imbalanced-category based skin lesion classification as well as spinal cord gray matter segmentation (which can be treated as pixel-wise classification), indicate that our proposed method can achieve much better generalization capability compared with other state-of-the-art baselines.  The code is available at \url{https://github.com/wyf0912/LDDG}.



\section{Related Works}

\textbf{Domain Adaptation and Generalization. }
To tackle the domain-shift problem between source and target domain data, traditional domain adaptation approaches focused on either subspace learning or instance re-weighting \cite{huang2006correcting,pan2011domain,zhang2015multi,ghifary2017scatter}.
Deep learning methods are also proved to be effective for domain adaptation task through either distribution alignment (e.g. Maximum Mean Discrepancy) \cite{long2015learning} or adversarial learning through feature level \cite{ganin2016domain,tzeng2017adversarial} or pixel level \cite{bousmalis2017unsupervised}. Recently, it has been shown that by considering pixel level adaptation and feature level adaptation together, better adaptation performance can be achieved \cite{hoffman2017cycada,li2020unsupervised}. 

Compared with domain adaptation, domain generalization is much more challenging, as we assume that we have no access to the target domain. Instead, we aim to train a model that is expected to be generalized to the ``unseen" target by assuming that only multiple source domains are available. For example, Yang and Gao~\cite{yang2013multi} proposed to leverage Canonical Correlation Analysis (CCA) to extract shareable information among domains. Muandet~\textit{et al.}~\cite{muandet2013domain} proposed a Domain Invariant Component Analysis (DICA) algorithm to learn an empirical mapping based on multiple source-domain data where the distribution mismatch across domains was minimized. This idea was further extended by \cite{li2018domain,li2020gmfad} in an autoencoder framework with distribution regularization on latent space.  In \cite{xu2014exploiting,li2017deeper}, the low-rank regularization based on classifier and model parameters were explored to extract universal feature representation.  Ghifary~\textit{et al.}~\cite{ghifary2015domain} proposed a multi-task autoencoder to learn domain invariant features by reconstructing the latent representation of a given sample from one domain to another domain. Motiian~\textit{et al.}~\cite{Motiian_2017_ICCV} proposed to minimize the semantic alignment loss as well as the separation loss based on deep learning models. Carlucci~\textit{et al.}~\cite{carlucci2019domain} proposed to shuffle the image patch to learn generalized feature representation. Recently, Wang~\textit{et al.}~\cite{wang2020heterogeneous} proposed to extend MixUp \cite{zhang2017mixup} to the settings of multiple domains for heterogeneous domain generalization task.  As for meta-learning based techniques, Li~\textit{et al.}~\cite{li2018learning} proposed to transfer the idea in \cite{finn2017model} to the ``unseen" target domain setting by randomly constructing meta-train and meta-test set, which was further extended by Balaji~\textit{et al.}~\cite{balaji2018metareg} with a scheme to learn a regularization network to improve the scalability of domain generalization. 

\textbf{Cross-Domain Medical Imaging Classification. }
Due to the various imaging protocols, device vendors and patient populations, we may also encounter the problem of distribution shift in clinical practice. To tackle such domain shift problem, image synthesis can be adopted through Generative Adversarial Networks \cite{goodfellow2014generative,zhu2017unpaired}  to mitigate the domain shift problem. For example, Zhang~\textit{et al.} \cite{zhang2018task} proposed to leverage CycleGAN for medical imaging problem to transfer the knowledge from CT images to X-ray images. Chen~\textit{et al.} \cite{chen2019synergistic} conducted domain translation from MR to CT domain for heart segmentation problem. With few label information available in target domain, Zhang~\textit{et al.}~\cite{zhang2018translating} proposed to conduct segmentation and data synthesis jointly to segment heart chambers in both CT and MR domain. Dou~\textit{et al.}~\cite{dou2018pnp} proposed a two parallel domain-specific encoders and a decoder where the weights are shared between domains to boost the performance of training on both single domain and cross-domain scenario. When target domain data are not available, Zhang~\textit{et al.} \cite{zhang2019unseen} proposed to conduct data augmentation on source domain to achieve better generalization capability for medical imaging classification task. Yoon~\textit{et al.} \cite{yoon2019generalizable} proposed to learn generalized feature representation through classification and contrastive semantic alignment technique \cite{Motiian_2017_ICCV}. More recently, Dou~\textit{et al.} \cite{dou2019domain} proposed to conduct meta-learning with global class alignment as well as local sample clustering regularization for medical imaging classification task.

\section{Methodology} 

\textbf{Preliminary. } We denote the training samples from multiple source domains on a joint space $\mathcal{X} \times \mathcal{Y}$ as $\mathcal{D}=\{(x_i^k,y_i^k)\}_{i=1}^{N_k}, k \in \{1,2,...,K\}$, where $x_i^k$ denotes the $i$th training sample from the $k$th source domain, $N_k$ is the number of samples in the $k$th domain, and  $y_i^k$ is the corresponding label groundtruth. The goal of domain generalization is that given a sample $x_T$ from an unseen domain, we aim to predict its output $\hat{y}^T$ through a trained classifier. 

We provide a framework named Linear-Dependency Domain Generalization (LDDG) that improves the generalization capability of medical imaging classification. By assuming that there exists linear dependency in the latent feature space among various domains based on a certain task, we propose to regularize the latent feature space by modeling intra-class variation among multiple source domains through rank constraint meanwhile matching the distribution of latent features extracted from multiple source domains to a pre-defined distribution prior, such that the shareable information among domains can be learned.   
The details of our proposed method are introduced below.

\textbf{Linear-Dependency Modeling. }
Directly training a classification network with a task-specific loss (e.g. cross-entropy loss) may not be feasible, as in the field of medical imaging, it is difficult to collect large and diverse datasets, which can lead to poor generalization on a new and ``unseen" domain. 
To improve the generalization capability of medical imaging classification, in \cite{zhang2019unseen}, based on the observation that medical image domain variability is more compact compared with other image data, a data augmentation approach was proposed based on three different aspects: image quality (e.g., blurriness), image appearance (e.g., brightness) and spatial shape (e.g., rotation, scaling), by assuming other characteristics are supposed to be  more consistent.  However, we empirically find that it is challenging to choose a suitable augmentation type as well as its magnitude for a specific medical imaging classification task. 

Inspired by the observation that most of the aforementioned augmentation processes can be conducted through linear transformation, we assume that there exists linear dependency on the latent feature space. To be more specific, by assuming that we have a medical image batch collected from $K$ different  domains with label $c$ as $\{x_{i_1,c}^1,x_{i_2,c}^2,...,x_{i_K,c}^K\}$, there exists a set of parameters $\{\alpha_1,\alpha_2,...,\alpha_K\}$ such that the corresponding latent features  $\{z_{i_1,c}^1,z_{i_2,c}^2,...,z_{i_K,c}^K\}$ hold the property that $z_{i_j,c}^j = \alpha_1 z_{i_1,c}^1 + \alpha_2 z_{i_2,c}^2 + ... + \alpha_{j-1} z_{i_{j-1},c}^{j-1} + \alpha_{j+1} z_{i_{j+1},c}^{j+1} + ... + \alpha_{K} z_{i_{K},c}^{K}$ for different $j$. In other words, there exists a dominant eigenvalue capturing the category information of the matrix $[z_{i_1,c}^1,z_{i_2,c}^2,...,z_{i_K,c}^K]$. Therefore, given a sample mini-batch  denoted by $\mathcal{X} = \{x_i^k\}$, we can obtain the corresponding latent features as $\mathcal{Z}$ through a posterior $q(z|x)$ parameterized by an encoder. By further conducting mode-1 flattening $\mathcal{Z}$ as $\mathbf{Z}$\footnote{We assume that the first dimension is associated with sample index.}, our proposed rank regularization can be given as $rank(\mathbf{Z}) = C$,
where $C$ is the number of categories of a specific task.

Setting the rank of $\mathbf{Z}$ to $C$ is equivalent to minimize the $(C+1)$th singular value of  $\mathbf{Z}$. By denoting it as $\sigma_{C+1}$, we can reformulate the rank loss and compute its sub-gradient as 
\begin{eqnarray}\label{eq:rank}
\mathcal{L}_{rank} = \sigma_{C+1}, \qquad \frac{\partial \sigma_{C+1}}{\partial \mathbf{Z}} = \mathbf{U}_{:,C+1}\mathbf{V}_{:,C+1}^\top,
\end{eqnarray} 
where $\mathbf{U}$ and $\mathbf{V}$ are obtained through SVD as $\mathbf{Z} = \mathbf{U}\mathbf{\Sigma}\mathbf{V}^\top$.

Noted that it is quite common to impose low-rank regularization in the final objective for domain generalization task \cite{xu2014exploiting,li2017deeper}. Our proposed method is different from these methods in two folds, 1) we impose rank regularization based on the latent feature space while the existing works imposed low-rank regularization on classifier parameters, which are not computational efficiency; 2) we set the rank to be a specific number instead of simply conducting low-rank regularization, we show in the experimental section that it can lead to better performance. 

\textbf{Distribution Alignment. }
In addition to modeling the linear dependency in latent feature space, we further propose to extract shareable information among multiple domains, such that a more transferable feature representation can be learned which can benefit generalization capability of deep neural networks. Existing techniques aimed to either minimize domain variance through distribution alignment between domain pairs \cite{muandet2013domain,li2018domain} or conduct local sample clustering through contrastive loss or triplet loss \cite{Motiian_2017_ICCV,dou2019domain}. However, based on our empirical analysis, the aforementioned technique may suffer from overfitting problem to the source domains, which is not surprising as the distribution of ``unseen" target domain may not match the distribution of multiple source domains. Thus, simply minimizing domain variance or local sample cluttering on source domains may not be able to generalize well to the ``unseen" one.  To this end, we propose to conduct variational encoding \cite{kingma2013auto} by adopting Kullback-Leibler (KL) divergence, which aims to match the latent features from multiple source domains to a pre-defined prior distribution. In our work, we adopt Gaussian distribution $\mathcal{N} \sim (0,1)$  as the prior distribution, which is computationally tractable through reparameterization trick \cite{kingma2013auto}. The KL divergence can be formulated as $KL(q(\mathcal{Z}|\mathcal{X})||\mathcal{N} \sim (0,1))$,  where $\mathcal{Z}$ is the latent features defined in the previous section. 
 We show in the next section that by jointly conducting linear-dependency modeling and distribution regularization through KL divergence can lead to an upper bound of empirical risk from any ``unseen" but related domains.  

\textbf{Theoretical Analysis. }
In this section, we provide the theoretical analysis of our proposed framework. In particular, We show that our framework can lead to an upper bound of expected loss on ``unseen" but related target domain.  We first make the following assumptions:


\begin{assumption}
	For any latent feature belong to domain $T$ with label $c$, it can be represented by data from other related domains, \textit{i.e.}, $q(z_{i_T,c}^T|x_{i_T,c}^T) = \sum_{j=1}^{K} \beta_j q(z_{i_j,c}^j|x_{i_j,c}^j)$, where $\beta_j >= 0$ and $\|\beta\| \leq M$ , $\{x_{i_1,c}^1,x_{i_2,c}^2,...,x_{i_K,c}^K\}$ belong to the same category as $x_T$.
\end{assumption}

Noted that Assumption 1 is a mild assumption in the field of medical imaging classification task and is also reasonable in our setting as we restrict the rank of latent features to be the number of category, such that there exists linear dependency based on the latent features belonging to the same category. 

We further make assumption on the loss function $\mathcal{L}$ based on the output of classifier.
\begin{assumption}
	(1) $\mathcal{L}$ is non-negative and bounded. (2) $\mathcal{L}$ is convex: $\mathcal{L}(\sum_j \lambda_j y_j,y) \leq \sum_j \lambda_j \mathcal{L}(y_j,y)$, where $\lambda_j \geq 0$ and $\sum_j \lambda_j=1$.
\end{assumption}

Note that this assumption is easy to be satisfied for several standard loss functions (e.g. cross-entropy loss).

Under these assumptions, we have the following theorems.  

\begin{theorem}
	Given a sample $x_{i_T,c}^T$ from target domain $T$ where the distribution of its latent variable is represented as $q(z_{i_T,c}^T|x_{i_T,c}^T) = \sum_{j=1}^{K} \beta_j q(z_{i_j,c}^j|x_{i_j,c}^j)$, its latent variable is within the manifold of $\mathcal{N} \sim (0,1)$.	
\end{theorem}
 
\begin{proof}
	For simplicity, we first denote the distribution on latent variables as $q_i(z),i=\{1,2,...,K,T\}$ as well as the Gaussian prior $\mathcal{N} \sim (0,1)$  as $q_*(z)$.  We can obtain the following upper bound,
	\begin{eqnarray}
	\!\!&\!\! \!\!&\!\! KL(q_T(z)||q_*(z))  =  \sum_{j=1}^{K}\beta_j \int_z q_j(z) \log \frac{q_T(z)}{q_*(z)}dz \nonumber \\
	\!\!&\!\! = \!\!&\!\! \sum_{j=1}^{K}\beta_j \int_z q_j(z) \log \frac{q_j(z)[1+(q_T(z)/q_j(z)-1)]}{q_*(z)}dz \leq \sum_{j=1}^{K} \beta_j KL(q_j(z)\|q_*(z)), 
	\end{eqnarray}
	where we use $\log (1+x) \leq x$ and $\int q_T(z)dz = \int q_j(z)dz = 1$. As $KL(q_j(z)\|q_*(z))$ is minimized according to our proposed distribution alignment, $KL(q_T(z)||q_*(z))$ is then minimized.
	
	This completes the proof.
\end{proof}

Theorem 1 shows that the latent feature of any unseen but related domain lies in the manifold of pre-defined prior. With the help of Theorem 1, we can further derive the upperbound of empirical risk of target domain.

\begin{theorem}
	Given data from $K$ source domains,  where the empirical risk of domain $j$ is given as  $\mathcal{L}(\hat{y}^j,y) = \epsilon_j \leq \epsilon$,     the expected loss $\mathcal{L}(\hat{y}^T,y)$ is at most $M\epsilon + \log C$, where $C$ denotes the number of category given a task, if the classification layer is linear with softmax normalization trained by $\mathcal{L}$ which is a cross-entropy loss. 
\end{theorem}

\begin{proof}
Based on Theorem 1, we have $q_T(z)=q_*(z)$. Thus, we have the following upper bound,
	\begin{eqnarray}
	\!\!&\!\! \!\!&\!\!  \int_{z } \mathcal{L}(\hat{y}^T,y)q_T(z)dz   =  \int_{z }   \mathcal{L}(\sum_{j=1}^{K} \beta_j \hat{y}^j,y) q_*(z) dz = \int_{z }   \mathcal{L}(\|\mathbf{\beta}\| \sum_{j=1}^{K} \frac{\beta_j}{\|\mathbf{\beta}\|} \hat{y}^j,y) q_*(z) dz \nonumber \\
	\!\!&\!\!  \leq \!\!&\!\!  \sum_{j=1}^{K} \frac{\beta_j}{\|\mathbf{\beta}\|  } \int_{z }\mathcal{L}( \|\mathbf{\beta}\| \hat{y}^j,y) q_j(z) dz \leq \sum_{j=1}^K \frac{\beta_j}{\|\beta\|} M \epsilon + \log C = M \epsilon + \log C,
	\end{eqnarray}
	where $\mathcal{L}(\cdot)$ denotes the cross-entropy loss with softmax operation. Noted that we only adopt a linear layer for classifier, thus $\hat{y}^{T}=\sum_{j=1}^{K} \beta_j \hat{y}^{j}$ holds based on Assumption 1. We also utilize the bounds of Log-Sum-Exp function $f(a) \leq \max\{a_1,a_2,…,a_n\}+\log n$, where $f(a) = \log \sum_{i=1}^n \exp(a_i)$. In our work, $\{a_1,a_2,...a_n\}$ corresponds to the softmax output of $n$ different nodes. Thus, the second line of proof holds. 
	
	This completes the proof.
\end{proof}

Theorem 2 shows that our proposed method has a good generalization capability. If the empirical risks on source domains are small, the empirical risk on target domain is also expected to be small. 

\textbf{Model Training. }
Our proposed architecture consists of three part, a feature extractor $Q_{\theta}$, a variational encoding network $F_{\omega}$,  and a classification network $T_{\phi}$. Regarding the classification network, we only adopt a linear module (e.g. convolutional layer, linear layer) without any non-linear processing, such that our assumption can be satisfied.  Images $\mathcal{X} = \{x_i^k\}$ are first fed into the feature extractor $Q_{\theta}$ to obtain the latent features, and then the latent features are resampled \cite{kingma2013auto} through variational encoding network $F_{\omega}$,  finally the classification network $T_{\phi}$ outputs the corresponding prediction $\{\hat{y}_i^k\}$. A cross-entropy loss together with rank and distribution regularization to penalize the difference between $\{\hat{y}_i^k\}$ and the groundtruth label $\{{y}_i^k\}$, the distribution difference between latent features and the Gaussian prior as well as the rank of the latent features. In summary, our model can be trained by minimizing the following objective as
\begin{equation}
\mathcal{L}_{obj} = \sum_{i,k} \mathcal{L}_c(\hat{y}_i^k, {y}_i^k) + \lambda_1 \mathcal{L}_{rank} + \lambda_2 KL(q(\mathcal{Z}|\mathcal{X})||\mathcal{N} \sim (0,1)),
\end{equation}
where $\mathcal{L}_c(\hat{y}_i^k, {y}_i^k)$ denotes the cross-entropy loss with softmax operation, $\mathcal{L}_{rank}$ is the rank loss defined in Equation \ref{eq:rank}. 

\section{Experiments}
In this section, we evaluate our proposed method based on two different medical imaging classification tasks:  skin lesion classification task  and gray matter segmentation task of spinal cord. {The detail of architectures and experimental settings can be found in supplementary materials.}

\subsection{Skin lesion classification}
We adopt seven public skin lesion datasets, including HAM10000 \cite{tschandl2018ham10000}, Dermofit (DMF) \cite{dermofit}, Derm7pt (D7P) \cite{d7p}, MSK \cite{msk_uda_sonic}, PH2 \cite{ph2}, SONIC (SON) \cite{msk_uda_sonic}, and UDA \cite{msk_uda_sonic},  which contain skin lesion images  collected from different equipments. We follow the protocol in \cite{yoon2019generalizable} by choosing seven-category subset from these datasets, including melanoma (mel), melanocytic nevus (nv), dermatofibroma (df), basal cell carcinoma (bcc), vascular lesion (vasc), benign keratosis (bkl), and actinic keratosis (akiec).
Each dataset is randomly divided into 50\% training set, 20\% validation set and 30\% testing set, where the relative class proportions are maintained across dataset partitions. As suggested in \cite{yoon2019generalizable}, for each setting, we use one dataset from DMF, D7P, MSK, PH2, SON and UDA as target domain and the remaining datasets together with HAM10000 as source domains. We use a ResNet18 model \cite{he2016deep} pretrained on ImageNet as the backbone for our proposed method as well as other baselines. 




\begin{table}[!t]
  \centering
  \caption{Domain generalization results on the skin lesion classification task. We repeat experiment for 5 times for each technique and report the mean value and standard deviation.}
    \scalebox{0.85}{
    \begin{tabular}{c|ccccc}
    \toprule
    Target & DeepAll & MASF \cite{dou2019domain} & MLDG \cite{li2018learning} & CCSA \cite{yoon2019generalizable}& LDDG (Ours) \\ \hline
    DMF    & 0.2492$ \pm 0.0127 $  & 0.2692$ \pm 0.0146 $ & 0.2673$ \pm 0.0452 $ & 0.2763$ \pm 0.0263 $ & \bf{0.2793}$ \pm 0.0244 $ \\
    D7P    & 0.5680$ \pm 0.0181 $  & 0.5678$ \pm 0.0361 $ & 0.5662$ \pm 0.0212 $ & 0.5735$ \pm 0.0227 $ & \bf{0.6007}$ \pm 0.0208 $ \\
    MSK    & 0.6674$ \pm 0.0083 $  & 0.6815$ \pm 0.0122 $ & 0.6891$ \pm 0.0167 $ & 0.6826$ \pm 0.0131 $ & \bf{0.6967}$ \pm 0.0193 $ \\
    PH2    & 0.8000$ \pm 0.0167 $  & 0.7833$ \pm 0.0101 $ & 0.8016$ \pm 0.0096 $ & 0.7500$ \pm 0.0419 $ & \bf{0.8167}$ \pm 0.0096 $ \\
    SON    & 0.8613$ \pm 0.0296 $  & 0.9204$ \pm 0.0227 $ & 0.8817$ \pm 0.0198 $ & 0.9045$ \pm 0.0128 $ & \bf{0.9272}$ \pm 0.0117 $ \\
    UDA    & 0.6264$ \pm 0.0312 $  & 0.6538$ \pm 0.0196 $ & 0.6319$ \pm 0.0284 $ & 0.6758$\pm 0.0138 $ & \bf{0.6978}$ \pm 0.0110 $ \\ \hline
    Avg    & 0.6287 & 0.6460 & 0.6396 & 0.6438 & \bf{0.6697} \\ 
    \bottomrule
    \end{tabular}
    } 

\label{tab:skin_result}%
\end{table}%

\textbf{Results:} We compare our method with state of the art domain generalization methods, including MASF \cite{dou2019domain}, MLDG \cite{li2018learning}, and CCSA \cite{yoon2019generalizable}, which have shown the capability to generalize across domains in the field of medical imaging. We report the baseline results by tuning the hyper-parameters in a wide range. We also adopt the baseline by directly training the model with the classification loss, which is referred as ``DeepAll". The results are shown in Table \ref{tab:skin_result}.

As we can see, all the domain generalization based techniques can outperform the DeepAll by directly training on source domains with classification loss. Among the domain generalization methods, MASF can achieve relatively better performance compared with MLDG and CCSA, as it is built upon meta-learning mechanism by considering both global and local based domain alignment. Compared with all baselines, our proposed algorithm can achieve better performance in a clear margin, which is reasonable as our proposed training mechanism leverage the advantage of both data augmentation and domain alignment, which is less likely to suffer from overfitting problem. We can also observe that in some cases, all algorithms can perform relatively well, which we conjecture that the domain gap between source and target domain is relatively small. However, the performances are not desired in some cases (e.g. when using DMF as target domain), which may be due to the large domain gap between source and target domain. This observation is also consistent with the results reported in \cite{yoon2019generalizable}. We also experiment using the data augmentation based technique BigAug \cite{zhang2019unseen}  by considering a wide-range of augmentation spaces but find it cannot yield competitive performance and the results are omitted here for brevity. We conjecture the reason that the augmentation types may not be suitable for skin lesion classification task. 

In practice, it is also likely that we only have the data from one single domain during training. To further analyze the effectiveness of our proposed method under this scenario, we consider HAM10000 for training and the others for testing for skin lesion classification task. Besides directly training with classification loss (DeepAll), we also compare with CCSA \cite{yoon2019generalizable} and MixUp \cite{zhang2017mixup}. Noted that other domain generalization baselines are not applicable in this case, as they require multiple domains available to simulate domain shift. The results are shown in Table \ref{tab:single}. In most of the cases, our proposed method can outperform the DeepAll, CCSA  as well as MixUp. For CCSA, as the contrastive loss is applied on only one source domain while target domain is unseen, it is likely to suffer from overfitting problem. For MixUp, the combination is conducted only in a convex manner, which may not be able to generalize  well to the out-of-distribution target domain.

\begin{table}[t]
  \centering
  \caption{Domain generalization results with HAM10000 as source domain. }
    \scalebox{0.8}{
    \begin{tabular}{cccccccc}
    \toprule
         &  DMF   & D7P   & MSK   & PH2   & SON   & UDA & Average \\
    \midrule
    DeepAll & 0.3003     & 0.4972      &    0.1667   &  0.4945     &   0.5025    &  0.4945 & 0.4093\\
    CCSA \cite{yoon2019generalizable}  & 0.2762 & 0.5082  &  0.4652 & 0.4667  & 0.5275  & 0.5055 & 0.4582\\
    MixUp \cite{zhang2017mixup}& \textbf{0.3514} & 0.4029  & 0.3000 & 0.4333 & 0.6296 & 0.4615 & 0.4298\\
    Ours  & 0.2943 & \textbf{0.5191}  & \textbf{0.5087} & \textbf{0.5500} & \textbf{0.6949} & \textbf{0.5714} & \textbf{0.5231}\\
    \bottomrule
    \end{tabular}%
    }
  \label{tab:single}%
\end{table}%



\begin{table}[!t]
  \centering
  \caption{Domain generalization results on gray matter segmentation task. }
    \scalebox{0.7}{
    \subtable[DeepAll]{
    \begin{tabular}{cc|ccccc}
    \toprule
    source & target & DSC   & CC    & JI & TPR   & ASD \\
    \midrule
    2,3,4 & 1     & 0.8560 & 65.34 & 0.7520 & 0.8746 & 0.0809 \\
    1,3,4 & 2     & 0.7323 & 26.21 & 0.5789 & 0.8109  & 0.0992 \\
    1,2,4 & 3     & 0.5041 & -209  & 0.3504 & 0.4926 & 1.8661\\
    1,2,3 & 4     & 0.8775 & 71.92  & 0.7827 & 0.8888 & 0.0599\\
    \midrule
    \multicolumn{2}{c|}{Average} & 0.7425	& -11.4 & 0.6160 & 0.7667 &	0.5265
  \\
    \bottomrule
    \end{tabular}%
    }

    
    \subtable[Probabilistic U-Net \cite{kohl2018probabilistic}]{
    \begin{tabular}{cc|ccccc}
    \toprule
    source & target & DSC   & CC    & JI & TPR   & ASD \\
    \midrule
    2,3,4 & 1     & 0.8387 & 59.94 & 0.7276  & 0.8943 & 0.1820 \\
    1,3,4 & 2     & 0.8067 & 51.53 & 0.6778 & 0.7555 & 0.0580  \\
    1,2,4 & 3     & 0.5113 & -188  & 0.3550 & 0.5638 & 2.0866 \\
    1,2,3 & 4     & 0.8782 & 72.18 & 0.7833 & 0.8910 & 0.2183 \\
    \midrule
    \multicolumn{2}{c|}{Average} & 0.7587 & -1.09 & 0.6359 & 0.7762 & 0.6362  \\
    \bottomrule
    \end{tabular}%
    }

  }  
 
  \scalebox{0.7}{
    \subtable[MASF \cite{dou2019domain}]{
    \begin{tabular}{cc|ccccc}
    \toprule
    source & target & DSC   & CC    & JI & TPR   & ASD \\
    \midrule
    2,3,4 & 1     & 0.8502 & 64.22 & 0.7415 & 0.8903 & 0.2274 \\
    1,3,4 & 2     & 0.8115 & 53.04 & 0.6844 & 0.8161 & 0.0826 \\
    1,2,4 & 3     & 0.5285 & -99.3 & 0.3665 & 0.5155 & 1.8554 \\
    1,2,3 & 4     & \bf{0.8938} & \bf{76.14} & \bf{0.8083} & \bf{0.8991} & 0.0366  \\
    \midrule
    \multicolumn{2}{c|}{Average} & 0.7710 & 23.52 & 0.6502 & 0.7803 & 0.5505  \\
    \bottomrule
    \end{tabular}%
    }

    
    \subtable[MLDG \cite{li2018learning}]{
    \begin{tabular}{cc|ccccc}
    \toprule
    source & target & DSC   & CC    & JI & TPR   & ASD \\
    \midrule
    2,3,4 & 1     & 0.8585 & 64.57 & 0.7489 & 0.8520 & 0.0573 \\
    1,3,4 & 2     & 0.8008 & 49.65 & 0.6696 & 0.7696 & 0.0745 \\
    1,2,4 & 3     & 0.5269 & -108  & 0.3668 & 0.5066 & 1.7708 \\
    1,2,3 & 4     & 0.8837 & 73.60 & 0.7920 & 0.8637 & 0.0451 \\
    \midrule
    \multicolumn{2}{c|}{Average} & 0.7675 & 19.96 & 0.6443 & 0.7480 & 0.4869 \\
    \bottomrule
    \end{tabular}%
    }
  }  

  \scalebox{0.7}{
    \subtable[CCSA \cite{yoon2019generalizable}]{
    \begin{tabular}{cc|ccccc}
    \toprule
    source & target & DSC   & CC    & JI & TPR   & ASD \\
    \midrule
    2,3,4 & 1     & 0.8061 & 50.15 & 0.6801 & 0.8703 & 0.1678 \\
    1,3,4 & 2     & 0.8009 & 50.04 & 0.6687 & 0.8141 & 0.0939 \\
    1,2,4 & 3     & 0.5012 & -112  & 0.3389 & 0.5444 & 1.5480 \\
    1,2,3 & 4     & 0.8686 & 69.61 & 0.7684 & 0.8926 & 0.0449  \\
    \midrule
    \multicolumn{2}{c|}{Average} & 0.7442 & 14.45 & 0.6140 & 0.7804 & 0.4637  \\
    \bottomrule
    \end{tabular}%
    }

    
    \subtable[\textcolor{black}{LDDG (Ours)}]{
    \begin{tabular}{cc|ccccc}
    \toprule
    source & target & DSC   & CC    & JI & TPR   & ASD \\
    \midrule
    2,3,4 & 1     & \bf{0.8708} & \bf{69.29} & \bf{0.7753} & \bf{0.8978} & \bf{0.0411}\\
    1,3,4 & 2     & \bf{0.8364} & \bf{60.58} & \bf{0.7199} & \bf{0.8485} & \bf{0.0416} \\
    1,2,4 & 3     & \bf{0.5543} & \bf{-71.6} & \bf{0.3889} & \bf{0.5923} & \bf{1.5187}  \\
    1,2,3 & 4     & 0.8910 & 75.46 & 0.8039 & 0.8844 & \bf{0.0289} \\
    \midrule
    \multicolumn{2}{c|}{Average} & \textbf{0.7881} & \textbf{33.43} & \textbf{0.6720} & \textbf{0.8058} & \textbf{0.4076}  \\
    \bottomrule
    \end{tabular}%
    }
  }
\label{tab:gm_result}%
\end{table}%

\subsection{Spinal cord gray matter segmentation}
We then consider the task of gray matter segmentation of spinal cord based on magnetic resonance imaging (MRI) to evaluate our proposed method. In particular, we adopt the data from spinal cord gray matter segmentation challenge \cite{prados2017spinal}, which are collected from four different medical centers with different MRI systems (Philips Achieva, Siemens Trio, Siemens Skyra). The voxel size resolutions are ranging from $0.25 \times 0.25 \times 2.5$mm to $0.5 \times 0.5 \times 5$mm. To evaluate the generalization capability of our proposed method, we consider the data collected from one medical center as a domain, which leads to four different domain, namely "site1", "site2", "site3" and "site4", where one domain is adopted as target domain and the remaining are considered as source domains.  We adopt 2D-UNet \cite{ronneberger2015u} as the backbone network by considering the MRI axial slice as input\footnote{Noted that we have tried 3D-UNet as suggested in \cite{prados2017spinal} but find the performances are similar to 2D-UNet in cross-domain scenario.}. Due to the imbalance of the number of voxels belonging to spinal cord gray matter and background in the MRI image, we follow \cite{prados2017spinal} to consider a two-stage strategy in a coarse-to-fine manner: 1) segment the spinal cord area (where the groundtruth of spinal cord is available), 2) segment the gray matter area from the output of 1) for our proposed method as well as baselines for comparison. 

 \textbf{Results:} We compare our method with state of the art domain generalization methods, including MASF \cite{dou2019domain}, MLDG \cite{li2018learning}, CCSA \cite{yoon2019generalizable}, by tuning the baseline hyper-parameters in a wide range. Moreover, we also compare with the Probabilistic U-Net \cite{kohl2018probabilistic} which automatically learned a prior for medical imaging segmentation task. For quantitative evaluation, we use a number of metrics to validate the effectiveness of our method. In particular, the  metrics include three overlapping metrics: Dice Similarity Coefficient (DSC), Jaccard Index (JI) and Conformity Coefficient (CC); One statistical based metrics: Sensitivity (a.k.a. True Positive Rate (TPR))\footnote{We omit True Negative Rate in our case due to the imbalance of the number of voxels belonging to spinal cord gray matter and background in the MRI image.}; One distance based metric: Average surface distance (ASD), which are all performed in 3D. The results are shown in Table \ref{tab:gm_result}. 
 
 As we can observe, our method achieve the best results when using ``site1", ``site2" and ``site3" as target domain based on all metrics, and achieve the best results in ``site4" under ASD, which shows the effectiveness of our proposed method. Among all other domain generalization based methods, MASF can achieve better performance compared with CCSA and MLDG. Such results are consistent with the performance for skin lesion classification task. We also observe that probabilistic U-Net can achieve relatively better performance compared with the ``DeepAll" baseline by directly training on source domain with classification loss. However, it may still suffer from overfitting problem as the learned prior may not generalize well to ``unseen" target domain.    Again, we also evaluate the method proposed in \cite{zhang2019unseen} by considering a wide range of data augmentation parameters but find the results are not desired. We conjecture that the default augmentation types are not suitable for this task. Some results of adopting \cite{zhang2019unseen} can be found in the supplementary materials.

We further show some qualitative results in Figure~\ref{fig:waveform}. As we can see, while the ``DeepAll" baseline as well as other domain generalization based methods fail to segment the gray matter (e.g. when using ``site2" as target domain) or over-segment a large portion of gray matter by extending the segmentation maps to the white matter (e.g. when using ``site3" as target domain), our proposed method can generally achieve better performance compared with all the methods for comparison.

\begin{figure}[!t]
    \centering
    \includegraphics[width=12cm, trim=50 0 0 0]{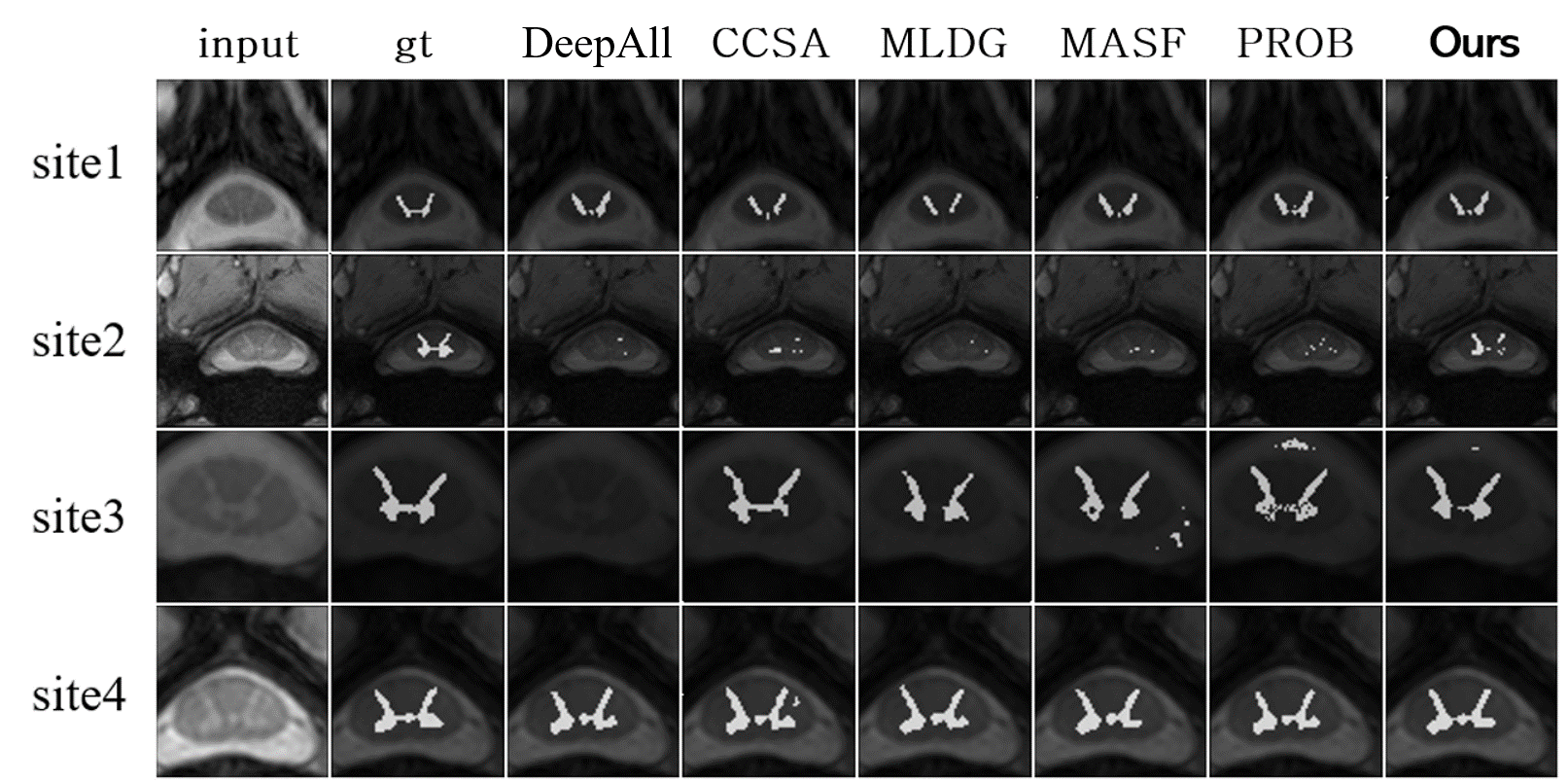}
    \caption{Qualitative comparisons. Each row represents a sample from a specific domain. Each column denotes the input, ground truth (gt) or different methods including DeepAll, CCSA \cite{yoon2019generalizable}, MLDG \cite{li2018learning}, MASF \cite{dou2019domain}, Probabilistic U-Net \cite{kohl2018probabilistic} (abbreviated as PROB here), respectively. As the area of interest in the original samples is very small,  all the samples are center cropped for better visualization.} 
    \label{fig:waveform}
\end{figure}

\subsection{Ablation Study}

we first conduct experiments on skin lesion classification task to understand the impact of different components of our proposed algorithm  by considering UDA as target domain. The results are shown in Table \ref{tab:abalation}, where ``Rank" and ``KL" denote our proposed rank regularization and distribution alignment through KL divergence minimization, respectively. ``LR" denotes the rank regularization with nuclear norm minimization, which is a popular way to conduct low-rank constraint. We have the following observations: 1) both rank regularization and distribution alignment through KL   can benefit the generalization capability for medical imaging classification task, which is reasonable as adopting these two terms jointly can theoretically lead to a empirical risk upper bound on target domain; 2) our proposed rank regularization can outperform the low-rank regularization by directly minimizing the nuclear norm, which is reasonable as the discriminative category specific information can be explored by enforcing the value of rank to be the number of category. 
\begin{figure}[!t]
\makeatletter\def\@captype{table}\makeatother
\begin{minipage}{.5\textwidth}
\centering
\scalebox{0.68}{
\begin{tabular}{ccccccc}
\toprule
\multicolumn{1}{c|}{Rank} & - & LR & LR & - &  \checkmark  &  \checkmark  \\
\multicolumn{1}{c|}{KL}      & - & - & \checkmark &  \checkmark  & - &  \checkmark  \\ \midrule
\multicolumn{1}{c|}{accuracy}  & 0.6264  & 0.6319 & 0.6703 & 0.6703  & 0.6813
  &  0.6978
 \\ \bottomrule
\end{tabular}}
\caption{Ablation study on key components of our method. We choose the skin lession classification task and use UDA as the target domain. }
\label{tab:abalation}
\end{minipage}
\qquad
\makeatletter\def\@captype{figure}\makeatother
\begin{minipage}{.45\textwidth}
\centering
\includegraphics[width=4.6cm, trim=-20 0 -20 0]{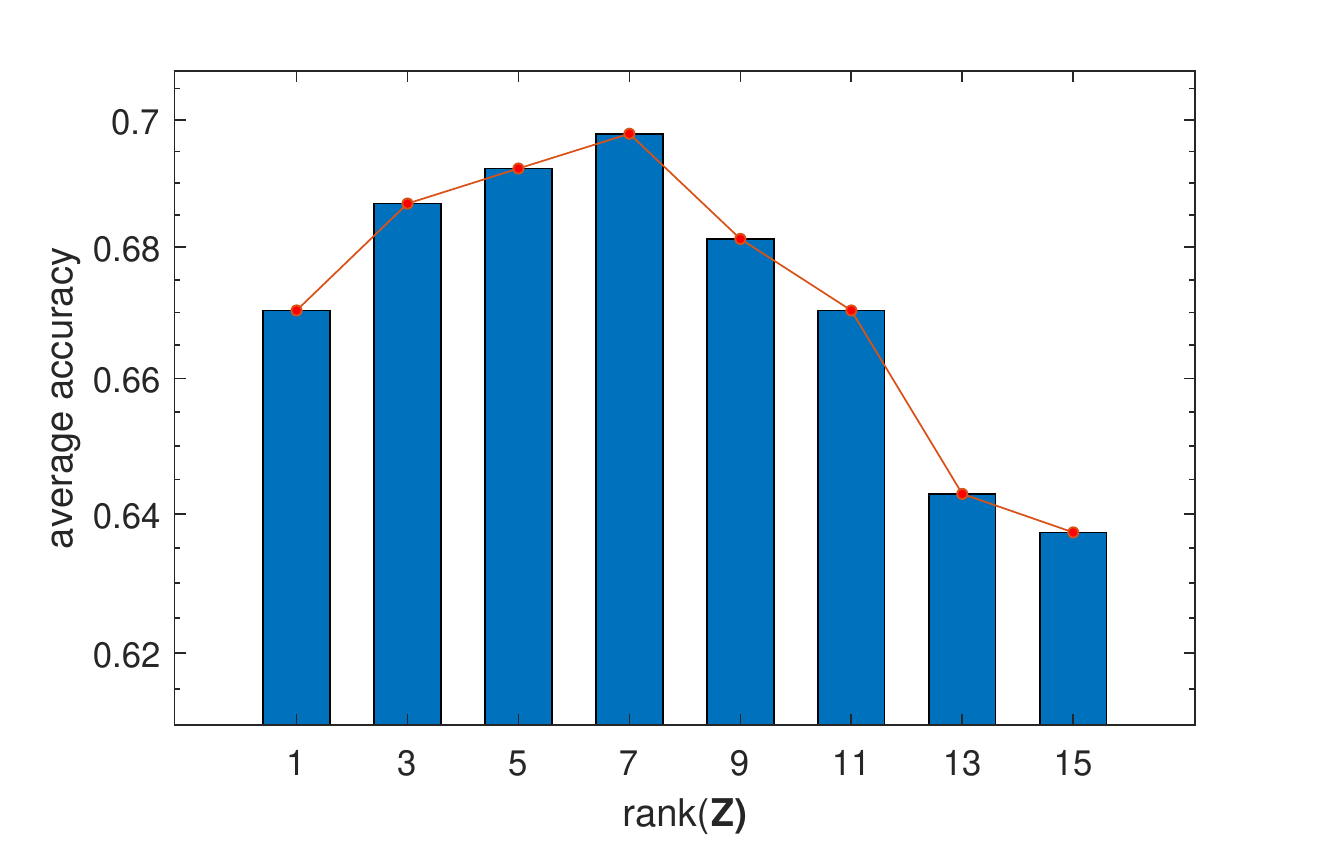}
\caption{The model performance with different $rank(\textbf{Z})$. UDA from skin lesion classification task is selected as the target domain.}
\label{fig:ablation}
\end{minipage}
\end{figure}

We then evaluate the effectiveness of our proposed rank regularization by varying the rank values of latent features by considering UDA as target domain. The results are shown in Figure \ref{fig:ablation}. As we can observe, the average classification accuracy has an ascending trend first, and then drop. In particular, the accuracy reaches its peak when $rank(\textbf{Z})=7$, which is also the number of category in our task. However, we also observe that the performance drops when $rank(\textbf{Z})$ gets larger, which is reasonable as increasing the value of $rank(\textbf{Z})$ may leads to noise information which can have negative impact to the task. Noted that we can still achieve better performance compared with only using low-rank regularization with nuclear norm as shown in Table \ref{tab:abalation}, which is reasonable as the nuclear norm does not take category information into consideration. 

Finally, we are interested in the singular values of our proposed method, which can be computed through SVD. In particular, we conduct experiments on segmentation task for stage 1 and stage 2 by considering ``site1" as the target domain. We show the  convergence results of singular values  in Figure~\ref{fig:rank}. As we can see, our proposed method can converge in 100 epochs for both stage 1 and 2 despite the fact that deep neural networks are highly nonlinear. Regarding the singular value, we find that the magnitude for $\sigma_1$ and $\sigma_2$ are relatively large while other values are much smaller, which is reasonable as we aim to explore the category specific information. Last but not the least, we find that compared with $\sigma_2$, the value of $\sigma_1$ is much larger, which we conjecture the reason that the area of gray matter (or spinal cord) is much smaller than others.  

\begin{figure}[h]
    \centering
    \includegraphics[width=12cm, trim=50 0 50 0]{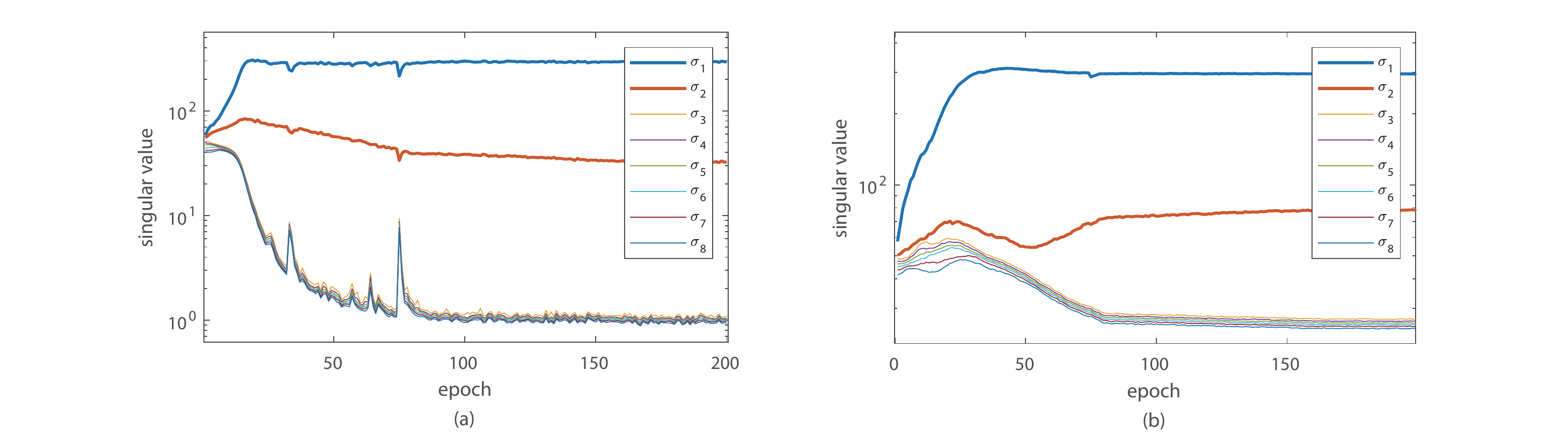}
    \caption{Analysis of singular values, (a) singular values in spinal cord segmentation stage (stage 1), (b) singular values of in gray matter segmentation stage (stage 2). }
    \label{fig:rank} 
\end{figure}


\section{Conclusion}
In this paper, we tackle the generalization problem in medical imaging classification task. Our proposed method takes the advantage of both linear-dependency modeling and domain alignment. In particular, we propose to learn a representative feature space for medical imaging classification task through variational encoding with linear-dependency regularization with a novel rank regularization term. Our theoretical analysis shows that an empirical risk upper bound on target domain can be achieved under our formulation. Experimental results on skin lesion classification task and spinal cord gray matter segmentation task show the effectiveness of our proposed method.  

\textbf{Broader Impact. }
Our proposed method shows reasonable potential in the application of clinically realistic environments especially under the scenarios where  only limited training samples are available and the capturing vendors and environments are diverse. In the short-term, the potential beneficiary of the proposed research lies in that it could significantly alleviate the domain shift problem in medical image analysis, as evidenced in this paper. In the long term, it is expected that the principled methodology could offer new insights in intelligent medical diagnostic systems. One concrete example is that the medical imaging classification functionality can be incorporated into different types of smartphones (with different capturing sensors, resolutions, etc.) to assess risk of skin disease (e.g. skin cancer in suspicious skin lesions) such that the terminal stage of skin cancer can be avoided. However, the medical data can be protected by privacy regulation such that the protected attributes (e.g. gender, ethnicity) may not be released publicly for training purpose. In this sense, the trained model may lack of fairness, or worse, may actively discriminate against a specific group of people (e.g. ethnicity with relatively small proportion of people). 
In the future, the proposed methodology can be feasibly extended to improve the algorithm fairness for numerous medical image analysis tasks and meanwhile guarantee the privacy of the protected attributes.   

\textbf{Acknowledgement. }The research work was done at the Rapid-Rich Object Search (ROSE) Lab, Nanyang Technological University. This research is supported in part by the Wallenberg-NTU Presidential Postdoctoral Fellowship,  the NTU-PKU Joint Research Institute, a collaboration between the Nanyang Technological University and Peking University that is sponsored by a donation from the Ng Teng Fong Charitable Foundation, the Science and Technology Foundation of Guangzhou Huangpu Development District under Grant 2017GH22 and 201902010028, and Sino-Singapore International Joint Research Institute (Project No. 206-A017023 and  206-A018001).

		\bibliographystyle{abbrv}
		\bibliography{nips}

\appendix

\section{Detail of Architectures and Experimental Settings}
\subsection{Experimental Setting for Skin Lesion Classification Task}
We use a ResNet18 model \cite{he2016deep} pretrained on ImageNet without the FC layer as the feature extractor $Q_\theta$ with the input size $224 \times 224$ for our proposed method as well as other baselines. 
For our method, the network before average pooling is used as the feature extractor. We insert a variational encoding network between the feature extractor and the final fully connected layer {which acts as the classifier $T_\phi$}.
The variational encoding network $F_{\omega}$ is implemented using two separated networks with the same architecture, including a fully connected layer with the output dimension as $512$, a Relu activation layer and a fully connected layer with the output dimension as $80$. We then conduct reparameterization trick to obtain the output of variational encoding network.  The classification network $T_\phi$ is a fully connected layer with the output size as $7$. 

For the hyperparameters, we choose $\lambda_1=0.001$ and $\lambda_2=0.4$ for all settings. Due to the class imbalance within and across datasets, we adopt the focal loss \cite{lin2017focal} as the classification objective for our proposed method as well as other baseline techniques. For implementation, the alternate form proposed in \cite{lin2017focal} is adopted as it is an extension of cross-entropy loss and is also bounded and convex, which satisfies our assumption in the manuscript. During training, the Adam optimizer is used with learning rate as $0.0001$, weight decay as $0.001$ and the size of minibatch as $32$.  We train the models for $200$ epochs and the learning rate is decreased by a factor $10$ after every $80$ epochs.  For evaluation on testing set, we use the best performing model on the validation set. 
\subsection{Experimental Setting for Spinal Cord Gray Matter Segmentation Task}
We adopt 2D-UNet \cite{ronneberger2015u} (without the last $1 \times 1$ convolutional layer) as the backbone network by considering the MRI axial slice as input. The variational encoding network $F_{\omega}$ is implemented using two separated network with an identical architecture which includes a latent layer using $1\times 1$ convolutional layer with output channel as $64$, a Relu activation layer and a $1\times 1$ convolution layer to predict mean and standard deviation layers of the distribution with output channel as $8$ . We then conduct reparameterization trick to obtain the output of variational encoding network.  {We further adopt a 1$\times$1 convolution layer as the classification network $T_\phi$ with the output channel size as $2$ for segmentation purpose. } 


For the hyperparameters, we use $\lambda_1=0.001$ and $\lambda_2=0.01$. We adopt the weighted binary cross-entropy loss for classification, where the weight of a positive sample is set to the reciprocal of the positive sample ratio in the region of interest. We use Adam algorithm with learning rate as 1e-4, weight decay as 1e-8 and the batch size as 8 for each domain for training. We train the model for 200 epochs, where the learning rate is decreased every 80 epoch with a factor of 10. For data processing, the 3D MRI data is first sliced into 2D in axial slice view and then center cropped to $160\times 160$. We further conduct random cropping which leads to the size as $144\times 144$ for training.



\section{BigAug \cite{zhang2019unseen} for Segmentation}
We present here by considering the data-augmentation based domain generalizing method BigAug \cite{zhang2019unseen}, which stacked different types of transformations, including sharpness, blurriness, noise, brightness, contrast, rotation, scaling, etc., by considering 2D-UNet for spinal cord gray matter segmentation task \cite{prados2017spinal}.  The results are shown in Table 1 (a). As we can observed, the performances are not desired by directly adopting the default parameters for augmentation in \cite{zhang2019unseen}, which are even worse than the ``DeepAll" baseline in terms of DSC and JI. 

To understand the reason why BigAug \cite{zhang2019unseen} with default parameter setting leads to negative transfer, we visualize in Figure~\ref{fig:bigaug} some examples of the transformed input and groundtruth pairs by considering both the groundtruth of spinal cord and gray matter. As we can see, by conducting the augmentation with default parameters in \cite{zhang2019unseen}, the quality of input deteriorates and the boundary can be oversmooth, which may further lead to more discrepancy between source and target domain. 

We further consider to adopt the same augmentation in \cite{zhang2019unseen} by tuning the parameters in a wide range to report the best segmentation performance for this task. The results are shown in Table 1 (b). We observe that there exists some improvement compared with ``DeepAll" baseline, but our proposed method can still outperform \cite{zhang2019unseen} with parameter tuning, which is reasonable as it is difficult to choose a suitable augmentation type and magnitude for different medical imaging classification tasks.


\begin{table}[!t]
  \centering
  \caption{Domain generalization results on gray matter segmentation task using BigAug \cite{zhang2019unseen}. }
    \scalebox{0.7}{
    \subtable[Default Parameters \cite{zhang2019unseen}]{
    \begin{tabular}{cc|ccccc}
    \toprule
    source & target & DSC   & CC    & JI & TPR   & ASD \\
    \midrule
    2,3,4 & 1     & 0.7675 & 38.47 & 0.6250 & 0.7798 & 0.1286 \\
    1,3,4 & 2     & 0.7542 & 34.50 & 0.6061 & 0.9187 & 0.1013  \\
    1,2,4 & 3     & 0.5468 & -76.2 & 0.3809 & 0.6381 & 1.9013 \\
    1,2,3 & 4     & 0.8706 & 70.18 & 0.7712 & 0.9232 & 0.0437 \\
    \midrule
    \multicolumn{2}{c|}{Average} & 0.7348 & 16.74 & 0.5958 & 0.8150 & 0.5437  \\
    \bottomrule
    \end{tabular}%
    }
  }
  \quad
  \scalebox{0.7}{
    \subtable[Tuned Parameters]{
    \begin{tabular}{cc|ccccc}
    \toprule
    source & target & DSC   & CC    & JI & TPR   & ASD \\
    \midrule
    2,3,4 & 1     & 0.8438 & 62.02 & 0.7334 & 0.8600 & 0.1613 \\
    1,3,4 & 2     & 0.7703 & 40.17 & 0.6269 & 0.8866 & 0.1802 \\
    1,2,4 & 3     & 0.5556 & -73.7 & 0.3905 & 0.6282 & 1.5560 \\
    1,2,3 & 4     & 0.8891 & 74.94 & 0.8009 & 0.8827 & 0.0362  \\
    \midrule
    \multicolumn{2}{c|}{Average} & 0.7647 & 25.86 & 0.6379 & 0.8144 & 0.4834 \\
    \bottomrule
    \end{tabular}%
    }


}
\label{tab:gm_result}%
\end{table}%

\begin{figure}[h]
    \centering
    \includegraphics[width=13cm, trim=50 0 0 0]{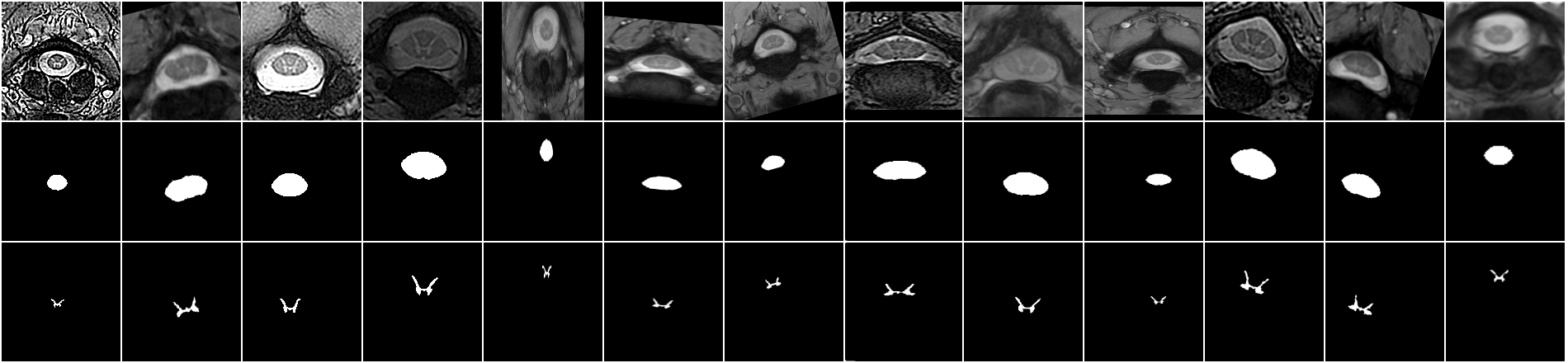}
    \caption{The samples of input and ground truth pairs generated from BigAug using the default hyper parameters. The first row shows the input, the second row shows the groundtruth of spinal cord, and the last row shows the groundtruth gray matters. } 
    \label{fig:bigaug}
\end{figure}

\end{document}